\documentclass[journal]{IEEEtran}

\usepackage{amsmath,amssymb,amsthm,graphicx}
\usepackage{url,algorithm,algorithmic}

\newcommand{\inner}[2]{\langle #1, #2 \rangle}
\newcommand{\g}[1]{\boldsymbol{#1}}

\newcommand{\R}[0]{\mathbb{R}} 
\newcommand{\E}[0]{\mathbb{E}} 
\newcommand{\N}[0]{\mathbb{N}} 
 
\newcommand{\I}[1]{\boldsymbol{1}_{#1}}
\renewcommand{\H}[0]{\mathcal{H}} 
\newcommand{\X}[0]{\mathcal{X}} 
\newcommand{\Y}[0]{\mathcal{Y}} 
\newcommand{\Z}[0]{\mathcal{Z}} 

\newcommand{\F}[0]{\mathcal{F}} 
\newcommand{\G}[0]{\mathcal{G}} 
\renewcommand{\L}[0]{\mathcal{L}}

\newtheorem{theorem}{Theorem}
\newtheorem{lemma}{Lemma}

\newtheorem{definition}{Definition}
\newtheorem{proposition}{Proposition}

\newcommand{\argmax}{\operatornamewithlimits{argmax}}

\begin{document}
%
\title{Error Bounds for Piecewise Smooth\\ and Switching Regression}

\author{Fabien~Lauer
\thanks{F. Lauer is with the LORIA, University of Lorraine, CNRS, Nancy, France
e-mail: (see https://members.loria.fr/FLauer/).}
}


\maketitle

\begin{abstract}
The paper deals with regression problems, in which the nonsmooth target is assumed to switch between different operating modes. Specifically, piecewise smooth (PWS) regression considers target functions switching  deterministically via a partition of the input space, while switching regression considers arbitrary switching laws. 
The paper derives generalization error bounds in these two settings by following the approach based on Rademacher complexities. For PWS regression, our derivation involves a chaining argument and a decomposition of the covering numbers of PWS classes in terms of the ones of their component functions and the capacity of the classifier partitioning the input space. This yields error bounds with a radical dependency on the number of modes. For switching regression, the decomposition can be performed directly at the level of the Rademacher complexities, which yields bounds with a linear dependency on the number of modes. By using once more chaining and a decomposition at the level of covering numbers, we show how to recover a radical dependency. Examples of applications are given in particular for PWS and swichting regression with linear and kernel-based component functions. 
\end{abstract}


%
\IEEEpeerreviewmaketitle

\section{Introduction}

The paper deals with regression problems, in which the nonsmooth target is assumed to switch between different operating modes. Specifically, we focus on two different (but related) settings: piecewise smooth (PWS) regression and switching regression. In PWS regression, the target function is assumed to switch between modes deterministically via a partition of the input space, while in switching regression the switchings can be arbitrary. 

Switching regression was introduced by \cite{Quandt58} and early algorithms include the one of \cite{Spath79} and the expectation-maximization methods of \cite{Hosmer74
,Desarbo88,Gaffney99}. 
%
Regression trees \cite{Breiman84}, and subsequent improvements \cite{Friedman91,Rao99}, are well-known early examples of piecewise regression models, together with the mixtures of experts \cite{Jacobs91}, which however usually consider smooth switchings. 
More recently, most of the work in this field was produced by the control community for hybrid dynamical system identification \cite{Paoletti07,LauerBook} and with a focus on optimization issues \cite{Vidal03,Ferraritrecate03,Roll04,Bemporad05,Juloski05b,Bako11,Lauer11a,Le11,Lauer13a,Pham14,Le14,Lauer18} and algorithmic complexity \cite{Lauer15,Lauer16}. 
This produced a number of practical methods for minimizing the empirical error of switching models with various optimization accuracy and computational efficiency. However, few results are available in terms of statistical guarantees for the obtained models, and most of them are established in a parametric estimation framework \cite{Jordan95,Bako11,Chen14} or under restrictive conditions on the target function \cite{Zeevi98}. 

Here, we aim at obtaining generalization error bounds for switching models in the agnostic learning framework \cite{KeaSchSel94}.  The tools we will use are those of statistical learning and we follow a standard approach to derive error bounds based on Rademacher complexities \cite{Koltchinskii02,Bartlett02,Mohri12}. While doing so, we particularly pay attention to the dependency of the obtained bound on the number of modes. In this respect, our work is related to recent discussions on the dependency of error bounds for margin multi-category classifiers on the number of categories, see, e.g., \cite{Kuznetsov14,Guermeur16,Musayeva17}. As in these works, a crucial role will be played by the decomposition of a global capacity measure as a function of capacities of component function classes. 

Specifically, we bound the Rademacher complexity of PWS classes using the chaining method \cite{Talagrand14} and covering numbers. We propose a decomposition scheme to express the covering numbers of a PWS class in terms of those of its component function classes and the capacity of the classifier defining the partition of the input space. For a large set of PWS classes, this results in error bounds with a radical dependency on the number of modes and efficient convergence rates when compared to the results of \cite{Guermeur16,Musayeva17} in multi-category classification.    

For switching regression, we follow a similar path but also consider a more straightforward approach, in which we apply the decomposition at the level of the Rademacher complexities themselves, without invoking covering numbers. A comparison of the two approaches shows that decomposing at the level of covering numbers is more advantageous with respect to the number of modes, with however a slightly worse rate of convergence with respect to the sample size for kernel-based classes. 

\paragraph*{Paper organization}
Section~\ref{sec:framework} formally exposes the two considered settings: PWS regression in Sect.~\ref{sec:pwsregression} and switching regression in~\ref{sec:switchingreg}. Then, Section~\ref{sec:errorbounds} derives error bounds for the PWS case in subsection~\ref{sec:errorpws} and for switching regression in subsection~\ref{sec:errorswitching}. Section~\ref{sec:conclusion} concludes the paper. Throughout the paper, a number of technical results are retained in Appendix, more precisely, in App.~\ref{sec:tools} for those from the literature and in App.~\ref{sec:proofreal}--\ref{sec:addresults} for newly derived ones. 

\paragraph*{Notation}
For an integer $n$, $[n]$ denotes the set of integers from $1$ to $n$. 
A bold lowercase letter with a subscript, $\g t_n$, denotes a sequence, $(t_i)_{1\leq i\leq n}$. Given two sets $\X$ and $\Y$, the set of functions from $\X$ into $\Y$ is written as $\Y^{\X}$.

\section{Theoretical framework}
\label{sec:framework}

Let $\X$ denote the input space and let the output space be $\Y =[-M,M]$ for some $M>0$. 
We assume that the relationship between inputs and outputs is characterized by the probability distribution of the random pair  $(X,Y) \in \X\times \Y$ and further assume that this distribution is unknown. 
Given a realization of the sample $\left((X_i,Y_i)\right)_{1\leq i\leq n}$ of $n$ independent copies of $(X,Y)$, the aim of regression is to learn the model $f$ that minimizes, over a certain function class to be defined below, the (expected) risk. In this paper, we define the risk from loss functions that can be clipped at $M$.

\begin{definition}[Clipping]
\label{def:clipping}
For any $M>0$ and $t\in\R$, we define the clipped version $\bar{t}$ of $t$ as
$$
	\bar{t} = \begin{cases}
		-M,& \mbox{if } t < -M \\
		t, & \mbox{if } t \in [-M,M] \\
		M, & \mbox{if } t > M.
		\end{cases}
$$
Similarly, the clipped version $\bar{f}$ of a function $f : \X\rightarrow \R$ is defined as
$$
	\forall x\in\X,\quad \bar{f}(x) = \overline{f(x)} = \begin{cases}
		-M,& \mbox{if } f(x) < -M \\
		f(x), & \mbox{if } f(x) \in [-M,M] \\
		M, & \mbox{if } f(x) > M
		\end{cases}
$$
and $\bar{\F}$ denotes the clipped function class $\{\bar{f} : f \in \F\}$.
\end{definition}

Recall from \cite{Steinwart08} that a loss function $\ell : \Y\times \R \rightarrow \R^+$  can be clipped at $M$ when, for all $(y,t)\in\Y\times \R$, 
$$
	\ell(y,\bar{t}) \leq \ell (y,t).
$$

\subsection{PWS regression}
\label{sec:pwsregression}

For PWS regression, we consider $\ell_p$-losses defined for $p\in[1,\infty)$ 
by $\ell_p(y,t) = |y - t|^p$ 
and the corresponding $\ell_p$-risks. 
\begin{definition}[$\ell_p$-risk and empirical $\ell_p$-risk]
\label{def:lploss}
For $p\in[1,\infty)$, the $\ell_p$-risk of a function $f$ from $\X$ into $\R$ is  
$$
	L_{p}(f) = \E_{X,Y} |Y - f(x)|^p
$$
and the corresponding empirical $\ell_p$-risk evaluated on an $n$-sample $(X_i, Y_i)_{1\leq i\leq n}$ is
$$
	\hat{L}_{p,n}(f) = \frac{1}{n} \sum_{i=1}^n  |Y_i - f(X_i)|^p.
$$
\end{definition}
Since the $\ell_p$-losses can be clipped at $M$, the $\ell_p$-risk of a clipped function, $L_{p}(\bar{f})$, is always smaller than the one of the unclipped $f$. The following thus considers that the final result of the learning procedure estimating $f$ is $\bar{f}$ and derives bounds on the risk of $\bar{f}$. 

We consider the agnostic learning framework and thus aim at uniform bounds on the $\ell_p$-risk holding (with high probability) for all $f$ in some predefined function class $\F$. In particular, we will focus on classes of piecewise smooth functions:
\begin{definition}[PWS class]
\label{def:pws}
Given a sequence $(\F_k)_{1\leq k\leq C}$ of classes of functions from $\X$ into $\R$ and a set of classifiers $\G$ from $\X$ into $[C]$, we define the PWS class of functions 
$$
	\F_{\G} = \left\{f \in \R^{\X} : f(x) = f_{g(x)}(x), \ g\in \G,\ f_k \in \F_k \right\}.
$$
\end{definition}

\subsection{Switching regression}
\label{sec:switchingreg}

Switching regression differs from PWS regression in the assumptions made regarding the switchings in the data generating process. While PWS regression assumes that the switchings are a deterministic function of $x$,\footnote{Note that it is not required for PWS regression to assume that the data are generated by a switching process: it can be considered merely as the use of a particular model class in a standard nonlinear regression setting.  However, it is mostly useful under such an assumption, in which case traditional regression methods based on smooth function classes may yield a larger error.} switching regression deals with arbitrary switchings.

In order to allow for such arbitrary switchings, we define classes of switching functions with vector-valued functions and embed the selection of the component function used to predict $Y$ in the definition of the loss functions and risks.
\begin{definition}[Switching class]
\label{def:switching}
Given a sequence $(\F_k)_{1\leq k\leq C}$ of classes of functions from $\X$ into $\R$, we define the switching class of vector-valued functions from $\X$ into $\R^C$ as 
$$
	\F^S = \left\{ f : f(x) = \left(f_k(x)\right)_{1\leq k\leq C}, \ f_k \in \F_k,\ 1\leq k\leq C\right\}.
$$
\end{definition}
\begin{definition}[$\ell_p$-switching risks]
\label{def:lplossswitching}
For $p\in[1,\infty)$, the switching $\ell_p$-risk of a vector-valued function $f$ from $\X$ into $\R^C$ is  
$$
	L_{p}^S(f) = \E_{X,Y}  \min_{k\in [C]} |Y - f_k(X)|^p
$$
and the corresponding switching empirical $\ell_p$-risk evaluated on an $n$-sample $(X_i, Y_i)_{1\leq i\leq n}$ is
$$
	\hat{L}_{p,n}^S(f) = \frac{1}{n} \sum_{i=1}^n \min_{k\in [C]} |Y_i - f_k(X_i)|^p .
$$
\end{definition}
Again, clipped versions of $f$, i.e., $\bar{f}= (\bar{f}_k)_{1\leq k\leq C}$,  will be used and it is easy to verify that the switching $\ell_p$-losses are clippable in the sense that
$$
	\forall (y,t) \in \Y\times \R^C,\quad  \min_{k\in [C]} |y - \bar{t}_k|^p \leq  \min_{k\in [C]} |y - t_k|^p .
$$

Such switching loss functions formalize the goal of accurately learning a collection of submodels so that, for all inputs, at least one submodel can predict the output well. Such a setting appears for instance in hybrid dynamical system identification \cite{LauerBook} and a number of computer vision applications \cite{Vidal07}.


\section{Error bounds}
\label{sec:errorbounds}

The main strategy for learning either PWS or switching models is empirical risk minimization, i.e., the minimization of the empirical risks $\hat{L}_{p,n}(f)$ or $\hat{L}_{p,n}^S(f)$ given a realization, $\left((x_i,y_i)\right)_{1\leq i\leq n}$, of the training sample. 
The following is dedicated to establishing upper bounds on the expected risks in terms of these empirical risks and a confidence (semi-)interval or control term depending on the function classes over which the minimization takes place.

We first introduce a general error bound based on the Rademacher complexity of the function class of interest. 
\begin{definition}[Rademacher complexity] 
\label{def:Rademacher-complexity}
Let $T$ be a random variable with values in $\mathcal{T}$. 
For $n \in \mathbb{N}^*$,
let $\g{T}_n = \left( T_i  \right)_{1 \leq i \leq n}$
be an $n$-sample of independent copies of $T$, let
$\boldsymbol{\sigma}_n = \left ( \sigma_i \right )_{1 \leq i \leq n}$
be a sequence of independent random variables uniformly distributed in $\{-1,+1\}$. 
Let $\mathcal{F}$ be a class of real-valued functions with domain $\mathcal{T}$.
The {\em empirical Rademacher complexity} of $\mathcal{F}$ given $\g{T}_n$ is
$$
\hat{\mathcal{R}}_n \left ( \mathcal{F} \right ) = 
\mathbb{E}_{\boldsymbol{\sigma}_n}
\left [ \left. \sup_{f \in \mathcal{F}} \frac{1}{n}
\sum_{i=1}^n \sigma_i f \left ( T_i \right )
\right| \g{T}_n \right ].
$$
The {\em Rademacher complexity} of $\mathcal{F}$ is
$$
\mathcal{R}_n \left ( \mathcal{F} \right ) 
= \mathbb{E}_{\mathbf{T}_n} 
\left [ \hat{\mathcal{R}}_n \left ( \mathcal{F} \right ) \right ]
= \mathbb{E}_{\mathbf{T}_n \boldsymbol{\sigma}_n} \left [
\sup_{f \in \mathcal{F}} \frac{1}{n}
\sum_{i=1}^n \sigma_i f \left ( T_i \right ) \right ].
$$
\end{definition}

\begin{theorem}[After, e.g., Theorem~3.1 in \cite{Mohri12}]
\label{thm:general}
Let $\mathcal{L}$ be a class of functions from $\mathcal{Z}$ into $[0,1]$ and $(Z_i)_{1\leq i\leq n}$ be a sequence of independent copies of the random variable $Z\in\mathcal{Z}$.  Then, for a fixed $\delta \in (0, 1)$,
with probability at least $1 - \delta$, uniformly over all $\ell \in \mathcal{L}$,
$$
\E_Z  \ell(Z) \leq \frac{1}{n}\sum_{i=1}^n \ell(Z_i) + 2 \mathcal{R}_n \left ( \mathcal{L} \right )
+ \sqrt{\frac{\log \frac{1}{\delta} }{2n}} .
$$
\end{theorem}

In the remaining of the paper, we assume without loss of generality that $Y\in[-M,M]$ with $M=\frac{1}{2}$, since otherwise we can recover this setting by rescaling $Y$. This choice is made in order to guarantee that the $\ell_p$-losses remain bounded by $1$ and that the corresponding function classes satisfy the assumptions of Theorem~\ref{thm:general}.

\subsection{Error bounds for PWS classes}
\label{sec:errorpws}

Our derivation of error bounds for PWS classes starts with the following consequence of Theorem~\ref{thm:general}, whose proof is given in Appendix~\ref{sec:proofreal}.
\begin{theorem}
\label{thm:real}
Let $\mathcal{F}$ be a real-valued function class. Then, for any $\delta \in (0, 1)$, 
with probability at least $1 - \delta$, the $\ell_p$-risk of Definition~\ref{def:lploss} is bounded uniformly $\forall \bar{f}\in \bar{\F}$ as
$$
	L_p(\bar{f}) \leq \hat{L}_{p,n}(\bar{f}) + 2p  \mathcal{R}_n \left ( \bar{\F}\right )
+ \sqrt{\frac{\log  \frac{1}{\delta}  }{2n}} . 
$$
\end{theorem}

Then, it remains to bound the Rademacher complexity of the clipped PWS class $\bar{\F}_{\G}$ which can be expressed from the clipped $\bar{\F}_k$'s as in Definition~\ref{def:pws}.

For this purpose, we will apply the chaining method \cite{Talagrand14} and introduce other capacity measures: the covering numbers.  

\begin{definition}[Pseudo-metric]
\label{def:pseudometric}
Given a sequence $\g t_n \in \mathcal{T}^n$, $d_{q,\g t_n}$ is the empirical pseudo-metric defined $\forall (f,f^\prime)\in\left(\R^{\mathcal{T}}\right)^2$ and $q\in [1,\infty)$ by
$$
	d_{q,\g t_n}(f,f^\prime) = \left( \frac{1}{n}\sum_{i=1}^n |f(t_i) - f^\prime(t_i)|^q \right)^{\frac{1}{q}}
$$
and for $q=\infty$ by
$$
	d_{\infty,\g t_n}(f,f^\prime) = \max_{i\in[n]} |f(t_i) - f^\prime(t_i)|.
$$
\end{definition}

\begin{definition}[Covering numbers] 
Given a function class $\F\subset\R^{\mathcal{T}}$ and a (pseudo-)metric $\rho$ over $\R^{\mathcal{T}}$, the {\em covering number} $\mathcal{N}(\epsilon, \F, \rho)$ at scale $\epsilon$ of $\F$ for the distance $\rho$ is the smallest cardinality of the proper $\epsilon$-net $\H\subseteq\F$ of $\F$ such that $\forall f\in\F$, $\rho(f,\H) < \epsilon$.
{\em Uniform covering numbers} are defined for all pseudo-metrics as in Definition~\ref{def:pseudometric} by
$$
	\mathcal{N}_q(\epsilon, \F, n) = \sup_{\g t_n\in \mathcal{T}^n} \mathcal{N}(\epsilon, \F, d_{q,\g t_n} ).
$$
\end{definition}

By considering covering numbers at different scales, chaining allows one to bound the Rademacher complexity of $\bar{\F}_{\G}$ whose diameter is $2M=1$ as follows (see Theorem~\ref{thm:chaining} in Appendix~\ref{sec:tools}): for any $N\in\N^*$,
\begin{equation}\label{eq:chaining}
	\hat{\mathcal{R}}_n(\bar{\F}_{\G}) \leq 2^{-N} + 6 \sum_{j=1}^N 2^{-j} \sqrt{\frac{\log\mathcal{N}(2^{-j} ,\bar{\F}_{\G}, d_{2,\g x_n})}{n}}. 
\end{equation}
The task is now to bound the covering numbers of the function class $\bar{\F}_{\G}$. This is done below by decomposing them in terms of the ones of the component function classes $\bar{\F}_k$ on the one hand and of the capacity of the classifier $\G$ on the other hand. In particular, we will measure the capacity of $\G$ with the growth function. 
\begin{definition}[Trace and growth function]
\label{def:growthfunction}
Let $\G$ be a set of classifiers from $\X$ to $[C]$. The {\em trace} of $\G$ on a set $\g x_n\in\X^n$ is the set
$$
	\G_{\g x_n} = \left\{ \left(g(x_1),\dots, g(x_n) \right) : g\in\G\right\} \subseteq [C]^n
$$
and the {\em growth function} of $\G$ is defined by
$$
\forall n\in \mathbb{N},\quad 	\Pi_{\G}(n) = \sup_{\g x_n \in \X^n} |\G_{\g x_n}| .
$$
\end{definition}

\subsubsection{Decomposition of the covering numbers}

The following gives two results based on two different techniques to optimize the dependency of the decomposition on the number of component functions (or modes), $C$. 
Note that these results are stated in terms of the clipped classes, but can be proved similarly for the unclipped ones. 
 
\begin{lemma}
\label{lem:coveringpwsLp}
Given a PWS class $\F_{\G}$ as in Definition~\ref{def:pws}, we have
$$
	\mathcal{N}(\epsilon, \bar{\F}_{\G}, d_{q,\g x_n}) \leq \Pi_{\G}(n)\prod_{k=1}^C \mathcal{N}\left(\frac{\epsilon}{C^{1/q}}, \bar{\F}_k, d_{q,\g x_n}\right) .	
$$
\end{lemma}
\begin{proof}
For each possible classification $\g c\in \G_{\g x_n}$ of $\g x_n$, let $g_{\g c}\in\G$ be a classifier from $\G$ producing this classification. Then, we build a set $H_{\g c}$ of functions $h\in\bar{\F}_{\G}$ such that $h(x_i) = h_{g_{\g c}(x_i)}(x_i) = h_{c_i}(x_i)$ with $(h_k)_{1\leq k\leq C}$ taken from the product of the smallest proper $\epsilon$-nets of the $\bar{\F}_k$'s. Since there are $\mathcal{N}(\epsilon, \bar{\F}_k, d_{q,\g x_n})$ functions $h_k$ in each one of these $\epsilon$-nets, we have
$$
	|H_{\g c}| \leq \prod_{k=1}^C \mathcal{N}(\epsilon, \bar{\F}_k, d_{q,\g x_n}) 
$$
and, since there are at most $\Pi_{\G}(n)$ classifications $\g c\in \G_{\g x_n}$, we can build a set $H = \bigcup_{\g c\in \G_{\g x_n}} H_{\g c} \subseteq\bar{\F}_{\G}$ with a cardinality bounded by
$$
	|H| \leq \sum_{\g c\in \G_{\g x_n}} |H_{\g c}| \leq \Pi_{\G}(n)\prod_{k=1}^C \mathcal{N}(\epsilon, \bar{\F}_k, d_{q,\g x_n}) .
$$

To conclude, we need to show that $H$ is a $(C^{1/q}\epsilon)$-net of $\bar{\F}_{\G}$ with respect to $d_{q,\g x_n}$. 
Given any $f\in\bar{\F}_{\G}$, there is some $\g c\in \G_{\g x_n}$ that coincides with the classification of $\g x_n$ by $g$ in $f(x)=f_{g(x)}(x)$, and thus for which, for $q<\infty$, for all functions $h\in H_{\g c} \subseteq H$, 
\begin{align*}
	d_{q,\g x_n}(f, h)^q &= \frac{1}{n} \sum_{i=1}^n |f(x_i) - h(x_i)|^q \\
	&= \frac{1}{n} \sum_{i=1}^n |f_{c_i}(x_i) - h_{c_i}(x_i)|^q\\
	&= \frac{1}{n} \sum_{k=1}^C \sum_{i : c_i = k}  |f_k(x_i) - h_k(x_i)|^q\\
	&\leq \sum_{k=1}^C  \frac{1}{n} \sum_{i=1}^n  |f_k(x_i) - h_k(x_i)|^q .
\end{align*}
By construction, among all the functions $h\in H_{\g c} \subseteq H$, there is at least one such that, for all $k\in[C]$, $h_k$ is the center of an $\epsilon$-ball containing $f_k\in \bar{\F}_k$, i.e., $\frac{1}{n} \sum_{i=1}^n |f_k(x_i) - h_k(x_i)|^q \leq \epsilon^q$. Thus, there is some $h\in H$ such that
\begin{align*}
	d_{q,\g x_n}(f, h)^q &\leq \sum_{k=1}^C \epsilon^q = C \epsilon^q .
\end{align*}
The statement for $q<\infty$ follows by rescaling $\epsilon$ by $1/C^{1/q}$. The case $q=\infty$ is proved similarly, but without the need for rescaling: 
\begin{align*}
	d_{\infty,\g x_n}(f, h) &= \max_{i\in[n]} |f(x_i) - h(x_i)| \\
	&= \max_{k\in[C]} \max_{i : c_i = k}  |f_k(x_i) - h_k(x_i)|\\
	&\leq \max_{k\in[C]} \max_{i \in[n]}  |f_k(x_i) - h_k(x_i)|\\
	&\leq \max_{k\in[C]} \epsilon = \epsilon.
\end{align*}
\end{proof}
Lemma~\ref{lem:coveringpwsLp} provides a bound on covering numbers in $L_q$-norm, which is most advantageous with respect to the dependency on $C$  for $q=\infty$. Covering numbers in $L_{\infty} $-norm can be used in chaining thanks to the following easy to verify inequality:
\begin{equation}\label{eq:coveringl2inf}
	\mathcal{N}(\epsilon, \bar{\F}_{\G}, d_{2,\g x_n}) \leq \mathcal{N}(\epsilon, \bar{\F}_{\G}, d_{\infty,\g x_n})  .
\end{equation}
However, this bound can be crude and not optimal in terms of the sample size $n$. Furthermore, Sauer-Shelah lemmas used to bound the covering numbers of the component classes $\bar{\F}_k$ can typically be made independent of $n$ for $q=2$ but not for $q=\infty$ (see Lemma~\ref{lem:sauerdimensionfree} below). Nonetheless, in some cases as emphasized in \cite{Musayeva17}, the relationship~\eqref{eq:coveringl2inf} can be sufficient to obtain a good dependency on both $C$ and $n$ in the final chained bound.

For comparison, the following lemma provides a bound with the same dependency on $C$ than the one in Lemma~\ref{lem:coveringpwsLp} for $q=\infty$ while relying solely on $L_2$-norm covering numbers. Its proof uses a slightly different technique based on the introduction of a collection of pseudo-metrics, all derived from the $L_2$-norm but based on different samples. The other ingredient is the non-increasing nature of uniform covering numbers of Glivenko-Cantelli (GC) classes \cite{Dudley91} with respect to $n$, proved in Appendix~\ref{sec:addresults}. Note that focusing on uniform GC classes is not very restrictive as this coincides with all learnable classes \cite{Alon97}.
\begin{lemma}\label{lem:coveringpws}
Given a PWS class $\F_{\G}$ as in Definition~\ref{def:pws} with uniform GC classes $\bar{\F}_k$, $1\leq k\leq C$, we have
$$
	\mathcal{N}(\epsilon, \bar{\F}_{\G}, d_{2,\g x_n}) 
	\leq \Pi_{\G}(n) \prod_{k=1}^C \mathcal{N}_2(\epsilon, \bar{\F}_k, n) .
$$
\end{lemma}
\begin{proof}
For each possible classification $\g c\in \G_{\g x_n}$ of $\g x_n$, we consider $C$ empirical pseudo-metrics $d_{x_i: c_i=k}$ defined as $d_{2,\g x_n}$ but on a restricted set of points of cardinality $n_k = \sum_{i=1}^n \I{c_i = k}$: $\forall (f,f')\in\left(\R^{\X}\right)^2$,
$$
	d_{x_i: c_i=k}(f,f') = \left(\frac{1}{n_k}\sum_{i : c_i = k} (f(x_i) - f'(x_i))^2 \right)^{\frac{1}{2}}.
$$
For each such distance, we build a proper $\epsilon$-net of $\bar{\F}_k$ of cardinality $\mathcal{N}(\epsilon, \bar{\F}_k, d_{x_i: c_i=k})$. 
Then, we build a set $H_{\g c}$ of functions $h$ such that $h(x_i) = h_{c_i}(x_i)$ with $(h_k)_{1\leq k\leq C}$ taken from the product of these $\epsilon$-nets, so that 
$$
	|H_{\g c}| \leq \prod_{k=1}^C \mathcal{N}(\epsilon, \bar{\F}_k, d_{x_i: c_i=k}),
$$
where the covering numbers depend on $\g x_n$ as usual, but also on $\g c$ via the definition of the pseudo-distances. 

Next, we consider a set $H = \bigcup_{\g c\in \G_{\g x_n}} H_{\g c}$ which contains at most 
$$
	|H| \leq \sum_{\g c\in \G_{\g x_n}} |H_{\g c}| \leq  \sum_{\g c\in \G_{\g x_n}}\prod_{k=1}^C \mathcal{N}(\epsilon, \bar{\F}_k, d_{x_i: c_i=k})
$$
functions.

Following the proof of Lemma~\ref{lem:coveringpwsLp}, for  any $f\in\bar{\F}_{\G}$, there is some $\g c\in  \G_{\g x_n}$ such that for all $h\in H_{\g c} \subseteq H$,
\begin{align*}
	d_{2,\g x_n}(f, h)^2 &= \frac{1}{n} \sum_{i=1}^n (f(x_i) - h(x_i))^2 \\
	&= \frac{1}{n} \sum_{i=1}^n (f_{c_i}(x_i) - h_{c_i}(x_i))^2\\
	&= \frac{1}{n} \sum_{k=1}^C \sum_{i : c_i = k}  (f_k(x_i) - h_k(x_i))^2\\
	&=  \sum_{k=1}^C \frac{n_k}{n} \frac{1}{n_k} \sum_{i : c_i = k}  (f_k(x_i) - h_k(x_i))^2\\	
	&=  \sum_{k=1}^C \frac{n_k}{n} d_{x_i : c_i = k }(f_k, h_k)^2.
\end{align*}
Thus, by the fact that $\sum_{k=1}^C \frac{n_k}{n} = 1$, $d_{2,\g x_n}(f, h)^2$ is expressed as a convex combination of squared sub-distances $d_{x_i : c_i = k }(f_k, h_k)^2$. These squared sub-distances being positive, their convex combination can be bounded by their maximum and we obtain
$$
	d_{2,\g x_n}(f, h)^2  \leq \max_{k\in[C]} d_{x_i : c_i = k}(f_k, h_k)^2 .
$$
By construction, among all the functions $h$ in $H_{\g c}$, there is at least one  such that, for all $k\in[C]$, $h_k$ is the center of an $\epsilon$-ball containing $f_k\in \bar{\F}_k$, i.e., $d_{x_i : c_i = k}(f_k, h_k)^2 \leq \epsilon^2$. Thus, there is some $h\in H$ such that
$$
	d_{2,\g x_n}(f, h)^2  \leq \epsilon^2,
$$
which proves that $H$ is an $\epsilon$-net of $\bar{\F}_{\G}$.

Now, we can improve the bound on $|H|$ by using uniform covering numbers. In particular, for any $\g c\in \G_{\g x_n}$, 
$$
	 \mathcal{N}(\epsilon, \bar{\F}_k, d_{x_i: c_i=k}) \leq\!\sup_{\g x_{n_k} \subset \X}\!\mathcal{N}(\epsilon, \bar{\F}_k, d_{2,\g x_{n_k}}) =  \mathcal{N}_2(\epsilon, \bar{\F}_k, n_k) .
$$
Thus, 
by using Lemma~\ref{lem:uniformcovmonotone} in Appendix~\ref{sec:addresults}:
\begin{align*}
	|H| &\leq  \sum_{\g c\in \G_{\g x_n}}  \prod_{k=1}^C \mathcal{N}_2(\epsilon, \bar{\F}_k, n_k) \\
	&\leq  \sum_{\g c\in \G_{\g x_n}}  \prod_{k=1}^C \mathcal{N}_2(\epsilon, \bar{\F}_k, n) \\
	&\leq  \Pi_{\G}(n) \prod_{k=1}^C \mathcal{N}_2(\epsilon, \bar{\F}_k, n) .
\end{align*}
\end{proof}

\subsubsection{Metric entropy bounds}

The decomposition results above readily yield general bounds on the metric entropy, $\log \mathcal{N}(\epsilon, \bar{\F}_{\G}, d_{2,\g x_n})$, of a PWS class $\bar{\F}_{\G}$ to be used in~\eqref{eq:chaining} and expressed in terms of Natarajan and fat-shattering dimensions.
\begin{definition}[Natarajan dimension]
Let $\G$ be a class of functions from $\X$ into $[C]$. A set $\{x_i\}_{i=1}^n\subset\X$ is said to be shattered by $\G$ if there exist two functions $a$ and $b$ from $\X$ into $[C]$ such that for every $i\in[n]$, $a(x_i) \neq b(x_i)$ and for every subset $I\subseteq [n]$, there is a function $g\in\G$ satsifying: $\forall i\in I$, $g(x_i) = a(x_i)$ and $\forall i\in[n]\setminus I$, $g(x_i)=b(x_i)$. The {\em Natarajan dimension} $d_{\G}$ of $\G$ is the maximal cardinality of a set $\{x_i\}_{i=1}^n\subset\X$ shattered by $\G$, if such maximum exists. 
Otherwise, $\mathcal{F}$ is said to have infinite Natarajan dimension. 
\end{definition}
\begin{definition}[Fat-shattering dimension \cite{KeaSch94}]
\label{def:fat}
Let $\mathcal{F}$ be a class of real-valued functions on $\X$.
For $\epsilon >0$, a set $\{ x_i\}_{i=1}^n\subset\X$ is said to be {\em ${\epsilon}$-shattered} by $\mathcal{F}$ if
there is a witness $\g b_n\in\R^n$ such that, for every subset $I\subseteq [n]$, there is a function
$f \in \mathcal{F}$ satisfying: 
$\forall i\in I$, $f(x_i) \geq b_i + \epsilon$ and $\forall i\in[n]\setminus I$, $f(x_i)\leq b_i - \epsilon$. 
The {\em fat-shattering dimension with margin $\epsilon$} of the class
$\mathcal{F}$, $d_{\F}(\epsilon)$, 
is the maximal cardinality of a set $\{ x_i\}_{i=1}^n\subset\X$
${\epsilon}$-shattered by $\mathcal{F}$, if such maximum exists.
Otherwise, $\mathcal{F}$ is said to have infinite
fat-shattering dimension with margin $\epsilon$.
\end{definition}

In particular, our first decomposition result in Lemma~\ref{lem:coveringpwsLp} yields the following metric entropy bound.
\begin{proposition}[PWS metric entropy bound 1]
\label{prop:metricentropy1}
Given a PWS class $\F_{\G}$, let $d_{\G}$ denote the Natarajan dimension of $\G$ and $d_{\F}(\epsilon) = \max_{k\in[C]} d_{\bar{\F}_k}(\epsilon)$ denote the pointwise maximum of the fat-shattering dimensions of the $\bar{\F}_k$'s. Then, for any $\epsilon\in(0,1]$ and $n\in\N^*$, 
$$
	\log \mathcal{N}(\epsilon, \bar{\F}_{\G}, d_{2,\g x_n}) \leq d_{\G}\log \frac{C e n}{2d_{\G}} + 6 C d_{\F}\left(\frac{\epsilon}{4}\right)\log^2 \frac{2 e n }{  \epsilon } .
$$
\end{proposition}
\begin{proof}
By application of~\eqref{eq:coveringl2inf} and Lemma~\ref{lem:coveringpwsLp} with $q=\infty$, we have
$$
	\log \mathcal{N}(\epsilon, \bar{\F}_{\G}, d_{2,\g x_n}) \leq \log \Pi_{\G}(n) + C \max_{k\in[C]} \log\mathcal{N}(\epsilon, \bar{\F}_k, d_{\infty,\g x_n}) .
$$
Then, we use two generalized Sauer-Shelah lemmas: Lemma~\ref{lem:sauernatarajan} in App.~\ref{sec:tools} to bound the first term with the Natarajan dimension and the result from \cite{Alon97} (in the form of Lemma~\ref{lem:alon} with $M=1/2$) to bound the last one in terms of the fat-shattering dimension. This yields
$$
\max_{k\in[C]} \log\mathcal{N}(\epsilon, \bar{\F}_k, d_{\infty,\g x_n}) \leq 2 d_{\F}\left(\frac{\epsilon}{4}\right)\log_2\frac{2 e n }{d_{\F}(\frac{\epsilon}{4}) \epsilon }  \log \frac{4 n }{\epsilon^2 } , 
$$
which, after using $d_{\F}(\epsilon/4)\geq 1$ and the relations $2/\log 2 <3$ and $\sqrt{n}<e n$, gives the result. If $d_{\F}(\epsilon/4) = 0$, the statement holds trivially since for all $k\in[C]$, $(f,f')\in\bar{\F}_k^2$ and $x\in\X$, $|f(x) - f'(x)|/2 < \epsilon/4$, which implies that $d_{\infty,\g x_n}(f,f')<\epsilon/2$ and thus that $\mathcal{N}(\epsilon, \bar{\F}_k, d_{\infty,\g x_n}) \leq \mathcal{N}(\frac{\epsilon}{2}, \bar{\F}_k, d_{\infty,\g x_n}) = 1$. 
\end{proof}

Conversely, our second decomposition result in Lemma~\ref{lem:coveringpws} yields the following bound. 
\begin{proposition}[PWS metric entropy bound 2]
\label{prop:metricentropy2}
Given a PWS class $\F_{\G}$ with uniform GC classes $\bar{\F}_k$, $1\leq k\leq C$, let $d_{\G}$ denote the Natarajan dimension of $\G$ and $d_{\F}(\epsilon) = \max_{k\in[C]} d_{\bar{\F}_k}(\epsilon)$ denote the pointwise maximum of the  fat-shattering dimensions of the $\bar{\F}_k$'s. Then, for any $\epsilon\in(0,1]$ and $n\in\N^*$,  
$$
	\log \mathcal{N}(\epsilon, \bar{\F}_{\G}, d_{2,\g x_n}) \leq d_{\G}\log \frac{C e n}{2d_{\G}} + 20 C d_{\F}\left(\frac{\epsilon}{96}\right) \log \frac{7}{\epsilon}.
$$
\end{proposition}
\begin{proof}
By application of Lemma~\ref{lem:coveringpws}, we have 
$$
	\log \mathcal{N}(\epsilon, \bar{\F}_{\G}, d_{2,\g x_n}) \leq \log \Pi_{\G}(n) + C \max_{k\in[C]} \log \mathcal{N}_2(\epsilon, \bar{\F}_k, n) .
$$
Then, we bound the first term as in Proposition~\ref{prop:metricentropy1} by Lemma~\ref{lem:sauernatarajan} and the last one with the dimension-free Sauer-Shelah lemma from \cite{Mendelson03} (in the form of Lemma~\ref{lem:sauerdimensionfree} with $M=1/2$).
\end{proof}

These two bounds share the same dependency on $C$, but the first one in Proposition~\ref{prop:metricentropy1} will lead to a worse dependency on $n$ when used in~\eqref{eq:chaining} for chaining. However, as discussed in \cite{Musayeva17}, this is not due to the dimension-free nature of the bound in Proposition~\ref{prop:metricentropy2} (which is independent of $n$). Indeed, the dependency on $n$ of the final chained bound is mostly impacted by how the metric entropy depends on $\frac{1}{\epsilon}$. Here, Proposition~\ref{prop:metricentropy2} exhibits a $O(\log \frac{1}{\epsilon})$ whereas the bound in Proposition~\ref{prop:metricentropy1} is in $O(\log^2 \frac{1}{\epsilon})$, which will translate into a $\sqrt{\log n}$ gain for the chained bound based on Proposition~\ref{prop:metricentropy2} (note however that due to the constants the true gain would be only visible in practice for very large $n$).

\subsubsection{Applications}

We now turn to specific examples of PWS classes. In particular, we focus on Euclidean input spaces $\X\subseteq \R^d$ with $d\geq 2$ and PWS classes constructed from a set of linear classifiers, i.e., for which 
\begin{equation}\label{eq:Glinear}
	\G = \{ g \in[C]^{\X} : g(x) = \argmax_{k\in[C]}  \inner{w_k}{x},\ w_k \in \R^d \}. 
\end{equation}
In this case, the Natarajan dimension of $\G$ satisfies $(C-1)(d-1)\leq d_{\G}\leq Cd$ (see, e.g., Corollary 29.8 in \cite{Shalev14}) and Lemma~\ref{lem:sauernatarajan} yields 
\begin{equation}\label{eq:piGlinear}
	\log \Pi_{\G}(n) \leq Cd\log \left( \frac{n e C}{2(C-1)(d-1)}\right) \leq  Cd\log \left( 3n\right).
\end{equation}

\paragraph{General PWS classes} 
We start with  general PWS classes $\F_{\G}$ based on linear classifiers as in~\eqref{eq:Glinear} and component function classes $\F_k$ that satisfy a polynomial growth assumption on the fat-shattering dimension, as considered, e.g., in \cite{Mendelson02,Guermeur16,Musayeva17}: 
\begin{equation}\label{eq:fatpoly}
	\forall \epsilon>0, \quad d_{\F}(\epsilon) = \max_{k\in[C]} d_{\bar{\F}_k}(\epsilon) \leq \alpha \epsilon^{-\beta} 
\end{equation}
for some positive numbers $\alpha$ and $\beta$. For example, if the $\F_k$'s are implemented by neural networks with $l$ hidden layers,~\eqref{eq:fatpoly} can be satisfied with $\beta=2(l+1)$ \cite{Bartlett98}. 
Note that~\eqref{eq:fatpoly} implies that the $\bar{\F}_k$'s are uniform GC classes \cite{Alon97}. 

Using the general metric entropy bound of Proposition~\ref{prop:metricentropy2} with~\eqref{eq:piGlinear} and the assumption~\eqref{eq:fatpoly} in~\eqref{eq:chaining} leads to 
\begin{equation}\label{eq:chainSN}
	\hat{\mathcal{R}}_n(\bar{\F}_{\G}) \leq 2^{-N} + \frac{6 S_N}{\sqrt{n}}
\end{equation}
 with
\begin{align*}
	S_N &= \sum_{j=1}^N 2^{-j} \sqrt{d_{\G}\log \frac{C e n}{2d_{\G}} + 20 C d_{\F}\left(\frac{2^{-j}}{96}\right) \log (7 \cdot 2^{j} )} \\
		&\leq \sqrt{ C }\sum_{j=1}^N 2^{-j} \sqrt{ d \log (3 n) + 20\cdot 96^{\beta} \alpha 2^{j\beta} \log (2^{j+3})}.
\end{align*}
Therefore, the dependency on the number of modes $C$ is radical for all such PWS classes and the degree $\beta$ of the polynomial growth of the fat-shattering dimensions only influences the convergence rate in $n$ of the Rademacher complexity. This convergence rate can be specified as follows:
\begin{align*}
	S_N &\leq 2^{\frac{\beta}{2}}\sqrt{ C }\sum_{j=1}^N 2^{j(\frac{\beta}{2}-1)} \sqrt{\frac{ d \log (3 n)}{2^{j\beta}} + 20\cdot 96^{\beta} \alpha  \log (2^{j+3})} \\
	&\leq  2^{\frac{\beta}{2}}\sqrt{ C }\sum_{j=1}^N 2^{j(\frac{\beta}{2}-1)} \sqrt{\frac{ d \log (3 n)}{2^{\beta}} +  14\cdot 96^{\beta} \alpha  (j+3)} \\
	&\leq  2^{\frac{\beta}{2}}\sqrt{ C }\sum_{j=1}^N 2^{j(\frac{\beta}{2}-1)} \sqrt{\frac{ d \log (3 n)}{2^{\beta}} +  56\cdot 96^{\beta} \alpha N} .	
\end{align*}
By setting $N= \lceil\log_2 n^{\frac{1}{\beta} }\rceil 
\leq \frac{1}{\beta}\log_2 ( 2^{\beta} n)$, this gives, for all $\beta \geq 1$, 
\begin{align*}
	S_N &\leq \sqrt{C  \left[ d +   \frac{56\cdot 192^{\beta} \alpha}{\beta}\right]\log_2 ( 2^{\beta} n)} \sum_{j=1}^N 2^{j(\frac{\beta}{2}-1)} ,
\end{align*}
while $2^{-N}\leq n^{-1/\beta}$ for the first term of~\eqref{eq:chainSN}. 
Thus, for $\beta=2$, we obtain 
$$
	S_N < \sqrt{ C  \left[d + 28\cdot 192^2 \alpha\right]\log_2 ( 4n)} \frac{1}{2}\log_2 (4n)
$$
and, with respect to $n$, 
\begin{equation}\label{eq:convratepwsbeta2}
	\hat{\mathcal{R}}_n(\bar{\F}_{\G}) = O\left( \frac{\log^{\frac{3}{2}}n}{\sqrt{n}} \right) .
\end{equation}
For $\beta>2$, we have
\begin{align}
	\sum_{j=1}^{N} 2^{j\left(\frac{\beta}{2}-1\right)}  &= \frac{2^{\left(\frac{\beta}{2}-1\right)\left(N+1\right)}-2^{\left(\frac{\beta}{2}-1\right)}}{2^{\left(\frac{\beta}{2}-1\right)}-1} \nonumber
	 < \frac{2^{\left(\frac{\beta}{2}-1\right)\left(N+1\right)}}{2^{\left(\frac{\beta}{2}-1\right)}-1} \\
	& < \frac{4^{\left(\frac{\beta}{2}-1\right)}}{2^{\left(\frac{\beta}{2}-1\right)}-1}  n^{\left(\frac{1}{2}-\frac{1}{\beta}\right)}, \label{eq:series2j}
\end{align}
which gives  
\begin{equation}\label{eq:convratepwsbetaSUP2}
	\hat{\mathcal{R}}_n(\bar{\F}_{\G}) = O\left(  \frac{\sqrt{\log n}}{n^{\frac{1}{\beta}}}  \right) .
\end{equation}

Overall, the convergence rates obtained are similar to the ones derived in \cite{Guermeur16} for multi-category classification based on score function classes satisfying~\eqref{eq:fatpoly}. However, thanks to Lemma~\ref{lem:coveringpws}, the radical dependency on $C$ is more favorable than the dependency on the number of categories in the result of \cite{Guermeur16}. 

\paragraph{Kernel-based PWS classes}

Let $\H$ be a reproducing kernel Hilbert space (RKHS) of reproducing kernel $K$ \cite{Berlinet04} and consider PWS classes with component function classes from this RKHS:
\begin{equation}\label{eq:Fkernel}
	\F_k = \{ f_k \in\H : \|f_k\|_{\H} \leq R_{\H}\} .
\end{equation}
Since the covering numbers and the fat-shattering dimensions can only decrease when going from $\F_k$ to $\bar{\F}_k$, 
the bounds on $d_{\F_k}(\epsilon)$ given by Lemma~\ref{lem:fatlinear} in App.~\ref{sec:tools} with $\phi : x\mapsto K(x,\cdot)$ also apply to the clipped component function classes and we have
$$
	\forall \epsilon>0,\quad	d_{\F}(\epsilon) \leq R_x^2 R_{\H}^2 \epsilon^{-2},
$$
where $R_x = \sup_{x\in\X} \|K(x,\cdot)\|_{\H} =  \sup_{x\in\X} \sqrt{K(x,x)}$. 
Thus, for $\G$ as in~\eqref{eq:Glinear}, the results above are applicable with $\alpha=R_x^2 R_{\H}^2$ and $\beta=2$. This yields
\begin{align*}
\hat{\mathcal{R}}_n(\bar{\F}_{\G}) 		
		&\leq \frac{1}{\sqrt{n}} + 3\log^{\frac{3}{2}}_2 ( 4 n)\sqrt{ \frac{C}{n}   \left(d + 28\cdot 192^2 R_x^2 R_{\H}^2\right)} 
\end{align*}
and a radical dependency on $C$ with a convergence rate in $O\left( \log^{3/2}(n) / \sqrt{n} \right)$.

\paragraph{PWA classes}

Let $\F_{\G}$ be a piecewise affine (PWA) class corresponding to Definition~\ref{def:pws} with $\G$ as in~\eqref{eq:Glinear} and linear function classes
\begin{equation}\label{eq:Fklinear}
	\F_k = \{f_k\in \R^{\X} : f_k(x) = \inner{w_k}{x},\ \|w_k\|_2 \leq R_w\}.
\end{equation}
The general results above could be applied similarly to this case (with $\phi:x\mapsto x$ in Lemma~\ref{lem:fatlinear}) in order to yield a bound in the flavor of the one obtained for kernel-based PWS classes. A slightly more efficient approach uses estimates of covering numbers that do not involve the fat-shattering dimension, as those given by Theorem~3 in \cite{Zhang02}. Thus, instead of the metric entropy bounds of the previous subsection, one could use Lemma~\ref{lem:coveringpws} with such estimates in the chaining formula~\eqref{eq:chaining}. Yet, the convergence rate would remain in $O\left( \log^{3/2}(n) / \sqrt{n} \right)$. 

In fact, the metric entropy bound of \cite{Zhang02} is suitable for large dimensional cases as it only involves the logarithm of the input dimension $d$. Yet, since in our case the dimension already appears outside of log terms when bounding $\log\Pi_{\G}(n)$ by \eqref{eq:piGlinear}, we can use more simple results that depend linearly on $d$, but enjoy a much better dependence on $\epsilon$. In particular, the following can be easily derived from classical results on the covering of unit balls in $\R^d$ (see, e.g., Exercise 2.2.14 in~\cite{Talagrand14}):
\begin{equation}\label{eq:coveringlinearRn}
	\forall \epsilon\leq R_x R_w,\quad \log \mathcal{N}_{\infty}(\epsilon, \F_k, n) \leq d \log \frac{(2+R_w)R_x}{\epsilon}  .
\end{equation}
Using this in~\eqref{eq:chaining} with~\eqref{eq:coveringl2inf}, Lemma~\ref{lem:coveringpwsLp} and~\eqref{eq:piGlinear}, we obtain
\begin{align*}
	\hat{\mathcal{R}}_n (\bar{\F}_{\G}) 
		&\leq 2^{-N} + 6\sqrt{\frac{ C d}{n}} \sum_{j=1}^N 2^{-j} \sqrt{  \log \left(3n (2+R_w)R_x 2^j\right) } \\
		&< 2^{-N} + 6\sqrt{\frac{Cd }{n}   \log \left(3n (2+R_w)R_x 2^N\right) }.
\end{align*}
Setting $N= \left\lceil \log_2\sqrt{n}\right\rceil \leq  \log_2 (2\sqrt{n})$, then yields an improved convergence rate:
\begin{align*}
	\hat{\mathcal{R}}_n (\bar{\F}_{\G}) &< \frac{1}{\sqrt{n}} +  6\sqrt{\frac{Cd }{n}  \log \left(6 (2+R_w)R_x n^{3/2}\right) }  \\
		&=O\left( \sqrt{\frac{\log n}{n}} \right) .
\end{align*}

\subsection{Error bounds for switching regression}
\label{sec:errorswitching}

We now come back to the switching regression setting of Sect.~\ref{sec:switchingreg}. 
Here, the focus is on the switching loss class of functions from $\Z=\X\times \Y$ to $\R$, 
\begin{equation}\label{eq:LpFS}
	\L_{p,\F^S}^S = \{ \ell  \in [0,1 ]^{\Z} : \ell(x,y) = \min_{k\in[C]} |y - \bar{f}_k(x)|^p,\ \bar{f}\in \bar{\F^S} \},
\end{equation}
induced by the vector-valued function class $\F^S$ (Def.~\ref{def:switching}). Indeed, Theorem~\ref{thm:general} applied to $\L_{p,\F^S}^S$ yields, with probability at least $1-\delta$ and uniformly over $\bar{\F}^S$, the following bound on the $\ell_p$-switching risk of Definition~\ref{def:lplossswitching}:
\begin{equation}\label{eq:firstboundswitch}
	L_p^S(\bar{f}) \leq \hat{L}_{p,n}^S(\bar{f}) + 2 \mathcal{R}_n \left ( \L_{p,\F^S}^S \right )
+ \sqrt{\frac{\log \frac{1}{\delta} }{2n}} .
\end{equation}
To finalize the bound, it remains to estimate the Rademacher complexity $\mathcal{R}_n \left ( \L_{p,\F^S}^S \right )$. 
For this purpose, we consider two decomposition schemes: one at the level of Rademacher complexities and another one at the level of covering numbers.  

\subsubsection{Decomposition of the Rademacher complexity}

In the case of switching regression, we can decompose directly at the level of the Rademacher complexities, without requiring the chaining machinery and covering numbers. 
In particular, we can derive the following decomposition result relating the Rademacher complexity of $\L_{p,\F^S}^S$ to the ones of the component function classes $\F_k$ with a linear dependency on the number $C$ of modes.
\begin{theorem}
\label{thm:switching}
Let $\F^S$ be a vector-valued function class as in Definition~\ref{def:switching}. Then, the Rademacher complexity of $\L_{p,\F^S}^S$~\eqref{eq:LpFS} is bounded by
$$
	\mathcal{R}_n \left ( \L_{p,\F^S}^S \right ) \leq  p \sum_{k=1}^C  \mathcal{R}_n \left ( \F_k\right ) .
$$
\end{theorem}
\begin{proof}
Let us define the following classes of functions from $\Z=\X\times \Y$ to $\R$:
\begin{equation}\label{eq:Ek}
	\forall k\in [C],\quad \mathcal{E}_k = \{ e_k \in \R^{\Z} : e_k(x,y) = y-\bar{f}_k(x),\ \bar{f}_k \in \bar{\F_k}\}.
\end{equation}
By using the facts that for any $(a_k)_{1\leq k\leq C} \in\R^C$, $\min_{k\in[C]} a_k = -\max_{k\in[C]} -a_k$ and that $\sigma_i$ and $-\sigma_i$ share the same distribution, we have:
\begin{align*}
	\mathcal{R}_n&\left( \L_{p,\F^S}^S \right) \nonumber\\
	&= \E_{\g X_n \g Y_n \g \sigma_n} \sup_{\bar{f} \in \bar{\F^S}} \frac{1}{n} \sum_{i=1}^n \sigma_i \min_{k\in[C]} |Y_i - \bar{f}_k(X_i ) |^p \\
	&= \E_{\g X_n \g Y_n \g \sigma_n} \sup_{\bar{f} \in \bar{\F^S}}  \frac{1}{n} \sum_{i=1}^n - \sigma_i \max_{k\in[C]} -|Y_i -\bar{f}_k(X_i ) |^p \\
	&= \E_{\g X_n \g Y_n \g \sigma_n} \sup_{\bar{f} \in \bar{\F^S}}  \frac{1}{n} \sum_{i=1}^n \sigma_i \max_{k\in[C]} - |Y_i - \bar{f}_k(X_i ) |^p \\
	&= \E_{\g X_n \g Y_n \g \sigma_n} \sup_{(e_k \in \mathcal{E}_k)_{k\in[C]}}  \frac{1}{n} \sum_{i=1}^n \sigma_i \max_{k\in[C]} - |e_k(X_i,Y_i)|^p \\
	&\leq \sum_{k=1}^C \mathcal{R}_n( -|\mathcal{E}_k|^p ) ,
\end{align*}
where the inequality is obtained by application of Lemma~\ref{lem:mohri} in Appendix~\ref{sec:tools}. 
By taking into account the range of $|\mathcal{E}_k|$, i.e., $[0 , 1]$, and the Lipschitz constant of $\phi(u)=u^p$ for $u$ in that interval, we obtain by contraction (Lemma~\ref{lem:contraction}) that $\mathcal{R}_n( -|\mathcal{E}_k|^p) \leq p\mathcal{R}_n( \mathcal{E}_k )$. Then, following the last steps of the proof of Theorem~\ref{thm:real} (Appendix~\ref{sec:proofreal}) leads to
\begin{align*}
	\mathcal{R}_n( \mathcal{E}_k) \leq \mathcal{R}_n(\bar{\F_k})\leq \mathcal{R}_n(\F_k), 
\end{align*}
which completes the proof.
\end{proof}


For switching linear regression with $\X\subseteq\R^d$ and $\F_k$ set as in~\eqref{eq:Fklinear}, we can combine~\eqref{eq:firstboundswitch} with Theorem~\ref{thm:switching} and Lemma~\ref{lem:radlinear} in App.~\ref{sec:tools} to get the risk bound
\begin{equation}\label{eq:boundswitchinglinear}
	L_p^S(\bar{f}) \leq \hat{L}_{p,n}^S(\bar{f}) + 2p C \frac{R_xR_w}{\sqrt{n}}
+ \sqrt{\frac{\log  \frac{1}{\delta} }{2n}} . 
\end{equation}

For switching nonlinear regression based on component function classes
from an RKHS $\H$ of reproducing kernel $K$ as in~\eqref{eq:Fkernel}, a similar result holds with $R_x = \sup_{x\in\X} \sqrt{ K(x,x)}$ and $R_w=R_{\H}$. 

\subsubsection{Chaining and decomposition of the covering numbers}

In another context, namely, multi-category classification as studied in \cite{Kuznetsov14,Guermeur16,Musayeva17}, the decomposition of the Rademacher complexity in terms of those of the component function classes yields a linear dependency on the number of categories, while chaining and decomposition at the level of covering numbers allows one to obtain a radical dependency. We now evaluate the possibility of reducing the linear dependency on the number of modes $C$ of the bound in Theorem~\ref{thm:switching} with such an approach. 

Consider the risk bound~\eqref{eq:firstboundswitch} based on the Rademacher complexity of the real-valued class $\L_{p,\F^S}^S$ defined in~\eqref{eq:LpFS}. 
We can bound the covering numbers of this class thanks to the structural result of Lemma~\ref{lem:coveringmax} in Appendix~\ref{sec:addresults} as follows. 
\begin{lemma}\label{lem:coveringswitching}
Let $\F^S$ be a vector-valued function class as in Definition~\ref{def:switching}. Then, for $\L_{p,\F^S}^S$ as defined in~\eqref{eq:LpFS}, the following holds for any $q\in[1,\infty)\cup\{\infty\}$:
$$
	\mathcal{N}(\epsilon, \L_{p,\F^S}^S, d_{q,\g z_n}) \leq \prod_{k=1}^C \mathcal{N}\left(\frac{\epsilon}{p C^{1/q}}, \bar{\F}_k, d_{q,\g x_n}\right).
$$
\end{lemma}
\begin{proof}
Let $\mathcal{E}_k$ be as in~\eqref{eq:Ek} and $\mathcal{E}_k^p$ denote the class $\{|e_k|^p  : e_k \in \mathcal{E}_k\}$. Then, $\L_{p,\F^S}^S$ is the pointwise minimum class $\{\min_{k\in[C]} e_k : e_k \in \mathcal{E}_k^p\}$ and Lemma~\ref{lem:coveringmax} gives
$$
	\mathcal{N}(\epsilon, \L_{p,\F^S}^S, d_{q,\g z_n}) \leq \prod_{k=1}^C \mathcal{N}\left(\frac{\epsilon}{C^{1/q}}, \mathcal{E}_k^p, d_{q,\g z_n}\right) .
$$
The contraction principle for covering numbers (see Lemma 27.3 in \cite{Shalev14}) 
implies that  
$$
	\mathcal{N}(\epsilon, \mathcal{E}_k^p, d_{q,\g z_n}) \leq \mathcal{N}\left(\frac{\epsilon}{p}, \mathcal{E}_k, d_{q,\g z_n}\right) .
$$
Since, for all pair of functions $e_k(x,y)=y-\bar{f}_k(x)$ and $e_k'(x,y)=y-\bar{f}'_k(x)$, $\forall (x,y)\in\Z=\X\times \Y$, $|e_k(x,y)-e_k'(x,y)| = |\bar{f}_k(x) - \bar{f}_k'(x)|$, we have $\forall \g z_n\in\Z^n$, $d_{q,\g z_n}(e_k,e_k') = d_{q,\g x_n}(\bar{f}_k, \bar{f}_k')$
and
$$
\mathcal{N}(\epsilon, \mathcal{E}_k, d_{q,\g z_n}) \leq \mathcal{N}(\epsilon, \bar{\F}_k, d_{q,\g x_n}) .
$$
Putting all these inequalities together concludes the proof.
\end{proof}

In order to optimize the dependency on $C$, we apply chaining (Theorem~\ref{thm:chaining} in App.~\ref{sec:tools}) to estimate the Rademacher complexity of $\L_{p,\F^S}^S$ with the relationship \eqref{eq:coveringl2inf} and covering numbers in $L_{\infty}$-norm controlled by Lemma~\ref{lem:coveringswitching}. This yields a radical dependency on $C$:
\begin{align}\label{eq:chainedboundswitching}
	\hat{\mathcal{R}}_n&(\L_{p,\F^S}^S) \\
	&\leq 2^{-N} + 6 \sum_{j=1}^N 2^{-j} \sqrt{\frac{\log\mathcal{N}(2^{-j} ,\L_{p,\F^S}^S, d_{2,\g z_n})}{n}} \nonumber \\
	&\leq 2^{-N} + 6  \sqrt{\frac{C}{n}}\sum_{j=1}^N 2^{-j}\sqrt{\max_{k\in[C]}\log\mathcal{N}\left(\frac{2^{-j}}{p},\bar{\F}_k, d_{\infty,\g x_n}\right)} \nonumber
\end{align}
and a convergence rate that depends on the capacity of the $\bar{\F}_k$'s as measured by their covering numbers.

In particular, for classes with fat-shattering dimensions that grow no more than polynomially with $\epsilon^{-1}$, as in~\eqref{eq:fatpoly}, Lemma~\ref{lem:alon} (Appendix~\ref{sec:tools}) yields 
$$
	\max_{k\in[C]}\log\mathcal{N}\left(\epsilon,\bar{\F}_k, d_{\infty,\g x_n}\right)
	\leq 6\cdot 4^{\beta} \alpha  \epsilon^{-\beta} \log^2 \frac{2 e n }{ \epsilon }
$$
and~\eqref{eq:chainedboundswitching} leads to
\begin{align*}
	\hat{\mathcal{R}}_n&(\L_{p,\F^S}^S) \\	
	&\leq 2^{-N} + 6  \sqrt{\frac{C}{n}}\sum_{j=1}^N 2^{-j}\sqrt{6\cdot 4^{\beta} \alpha  p^{\beta} 2^{j\beta} \log^2 (2 e n p 2^j)} \\
	&\leq 2^{-N} + 6\cdot 2^{\beta}  \sqrt{\frac{6\alpha p^{\beta} C}{n}  }\log (2 e n p 2^N )\sum_{j=1}^N 2^{j\left(\frac{\beta}{2}-1\right)}.
\end{align*}
Setting $N= \lceil\log_2 n^{\frac{1}{\beta} }\rceil \leq \frac{1}{\beta}\log_2 ( 2^{\beta} n)$, gives, for $\beta=2$, 
\begin{align*}
	\hat{\mathcal{R}}_n (\L_{p,\F^S}^S) &\leq \frac{1}{\sqrt{n}} + 12p \sqrt{\frac{6\alpha C }{n}  }\log (4 e p n^{\frac{3}{2}} ) \log_2 (4n) \\
		&\leq \frac{1}{\sqrt{n}} + 26 p \sqrt{\frac{\alpha C }{n}  }\log^2 (5pn ) \\ 
		& = O\left( \frac{\log^2 n}{\sqrt{n}}\right).
\end{align*}
For $\beta>2$, recalling \eqref{eq:series2j} leads to
\begin{align*}
	\hat{\mathcal{R}}_n (\L_{p,\F^S}^S) &\leq \frac{1}{n^{\frac{1}{\beta}}} +\frac{3\cdot 2^{2\beta-1}  \sqrt{ 6\alpha p^{\beta} C } }{2^{\left(\frac{\beta}{2}-1\right)}-1}  \frac{\log (4 e p n^{\frac{1}{\beta}+1} )}{n^{ \frac{1}{\beta} } } \\
	&= O\left( \frac{\log n}{n^{ \frac{1}{\beta}} } \right).
\end{align*}

Overall, we observe an additional factor in $O(\sqrt{\log n})$ compared to the bounds in~\eqref{eq:convratepwsbeta2} and~\eqref{eq:convratepwsbetaSUP2} for PWS regression with similar component function classes.

\paragraph{Switching kernel regression}
For kernel-based classes $\F_k$ as in~\eqref{eq:Fkernel}, we could apply the results above with $\beta=2$ thanks to Lemma~\ref{lem:fatlinear}. However, for such function classes, the covering numbers can be more efficiently bounded without invoking a Sauer-Shelah lemma and the fat-shattering dimension. In particular, for $L_{\infty}$-norm covering numbers, we can use Lemma~\ref{lem:zhanglinearinf} in~\eqref{eq:chainedboundswitching} and obtain, with $N =\lceil\log_2 \sqrt{n})\rceil \leq \log_2 ( 2 \sqrt{n})$,
\begin{align*}
\hat{\mathcal{R}}_n &(\L_{p,\F^S}^S) \\&\leq \frac{1}{\sqrt{n}} + 36p R_x R_{\H}\log_2(2\sqrt{n}) \sqrt{\frac{C}{n}\log(30p  R_x R_{\H} n^{3/2})}\\
	&= O\left(\frac{\log^{\frac{3}{2}} n}{\sqrt{n}}\right).
\end{align*}
Thus, for switching kernel regression, the bound is essentially of the same order as the one for kernel-based PWS regression. 

Compared with~\eqref{eq:boundswitchinglinear}, we gained a $\sqrt{C}$ but also introduced a $\log^{3/2} n$ factor, which is only beneficial when $\log^3 n < C < \sqrt{n}$ (up to constant factors). This limited range over which chaining provides a gain is due to the use of kernel-based classes whose Rademacher complexity can be very efficiently bounded.   

\paragraph{Switching linear regression} For switching linear regression with $\X\subset\R^d$ and classes $\F_k$ as in~\eqref{eq:Fklinear}, the convergence rate is much better and in fact similar to that of~\eqref{eq:boundswitchinglinear}. To see this, note that with~\eqref{eq:coveringlinearRn} we can apply the integral form of chaining (Theorem~\ref{thm:chaining}) to obtain 
\begin{align}
	\hat{\mathcal{R}}_n (\L_{p,\F^S}^S) &\leq \frac{12}{\sqrt{n}}\int_{0}^{1/2} 	\sqrt{\log \mathcal{N}( \epsilon, \L_{p,\F^S}^S, d_{\infty,\g z_n}) } \,d\epsilon \nonumber\\
	&\leq 12 \sqrt{\frac{C}{n}}\int_{0}^{1/2} 	\sqrt{\max_{k\in [C]}\log \mathcal{N}( \epsilon/p, \bar{\F}_k, d_{\infty,\g x_n}) } \,d\epsilon \nonumber\\
	&\leq 12 \sqrt{\frac{Cd}{n}}\int_{0}^{pR_w R_x} 	\sqrt{\log\left(\frac{p(2+R_w)R_x}{\epsilon}\right)} \,d\epsilon \nonumber\\
	&\leq 12 pR_w R_x \sqrt{\log\left(2/R_w+1\right)}  \sqrt{\frac{Cd}{n}}  \label{eq:radswitchedlinearcov}.
\end{align}

By comparing with Theorem~\ref{thm:switching} and~\eqref{eq:boundswitchinglinear}, we had to pay a $\sqrt{d}$ factor in exchange for a $\sqrt{C}$ one, which is advantageous in low or moderate-dimensional cases, as those that often occur in applications such as hybrid system identification \cite{Lauer18}.  

\section{Conclusions}
\label{sec:conclusion}

The paper derived error bounds for piecewise smooth and switching regression. These bounds are based on a decomposition of the capacity measure of the class of interest in terms of those of its component function classes. Different levels of decomposition were explored to optimize the dependency of the bounds on the number of components and a radical dependency was obtained for both PWS and switching regression via chaining and decomposition at the level of covering numbers. 
We note that this radical dependency is not a final characterization of the optimal growth rate. Indeed, the application of chaining could have been optimized (for instance by replacing $n$ by $n/\sqrt{C}$ when setting the value of $N$) to yield (only slightly) better growth rates at the cost of more complex expressions for the bounds.  

Open issues include the followings.

{\em Decomposition.} While we could also directly decompose the Rademacher complexity of a switching loss class, the efficient decomposition of PWS classes at the level of Rademacher complexities remains an open issue. Even if we can expect a worse dependency on the number of modes, as for the arbitrarily switching case, the convergence rates might be better for specific component function classes such as linear ones or RKHS balls. 
Decomposition can also be performed at a third level, namely, the one of fat-shattering dimensions. However, there are reasons to expect a quadratic dependency of the fat-shattering dimension of a PWS class on the number of modes, which, after taking the square root of the metric entropy, would result in a bound with linear dependency.

{\em Unbounded regression.} Error bounds can be derived for unbounded regression using assumptions on moments of the loss or non-constant envelopes as, e.g., in \cite{Pollard89,Cortes10}. Though our results were obtained for a bounded output space, all the decompositions of the capacity measures can be derived similarly for the unclipped (and unbounded) classes, which should allow for the extension to the unbounded case. 

{\em Non-independent case.} We assumed independence of the sampled data. Extending our results to the non-independent case would be of primary interest for time-series prediction and their application to hybrid dynamical system identification, where the input typically includes lagged values of the output. Works in that direction could follow the bounding schemes developed in \cite{Hang16}.

{\em Model selection.} Based on our results, practical procedures implementing structural risk minimization could be envisioned to tune the number of modes $C$. Indeed, the empirical risk could be minimized for a sequence of PWS or switching classes with increasing $C$, before selecting the model with lowest value of the error bound.

\appendices

\section{Technical results from the literature}
\label{sec:tools}

We recall the contraction principle for Rademacher complexities.
\begin{lemma}[After Theorem~4.12 in~\cite{Ledoux91}]\label{lem:contraction}
If $\phi: \R\rightarrow \R$ is a Lipschitz continuous function with Lipschitz constant $L_{\phi}$, i.e., if $\forall (u,v)\in \R^2,\ |\phi(u) - \phi(v)| \leq L_{\phi} |u-v|$, 
then 
$$
	\mathcal{R}_n ( \phi \circ \F ) \leq L_{\phi} \mathcal{R}_n ( \F ) ,
$$
where $\phi \circ \F$ denotes the class of functions $\phi \circ f$ with $f \in\F$.
\end{lemma}

The following chaining technique due to Dudley relates the Rademacher complexity to the covering numbers.
\begin{theorem}
\label{thm:chaining}
Let $\F$ be a real-valued function class over $\mathcal{T}$ and, for any $\g t_n\in\mathcal{T}^n$, let $D_{\F} = \sup_{(f,f')\in\F^2} d_{2,\g t_n}(f,f')$ denote its diameter. Then, for any $N\in \mathbb{N}^*$, 
$$
	\hat{\mathcal{R}}_n(\F) \leq \frac{D_{\F}}{2^{N} } + 6 D_{\F} \sum_{j=1}^N 2^{-j} \sqrt{\frac{\log\mathcal{N}(D_{\F}2^{-j} , \F, d_{2,\g t_n})}{n}} 
$$
and, if the integral exists, 
$$
	\hat{\mathcal{R}}_n(\F) \leq 12\int_0^{D_{\F}/2}\sqrt{\frac{\log\mathcal{N}(\epsilon , \F, d_{2,\g t_n})}{n}} d\epsilon.
$$
\end{theorem}

The following result upper bounds the Rademacher complexity of a class of functions defined as the pointwise maximum of a set of functions. 
\begin{lemma}[After Lemma 8.1 in \cite{Mohri12}]
\label{lem:mohri}
Let $(\mathcal{U}_k)_{1\leq k\leq K}$ be a sequence of $K$ classes of real-valued functions on $\Z$. Then, the class $\mathcal{U} = \{u \in\R^{\Z} : u(z) = \max_{k\in[K]} u_k(z) ,\ u_k \in \mathcal{U}_k \}$ has an empirical Rademacher complexity bounded by
$$
	\hat{\mathcal{R}}_n(\mathcal{U}) \leq \sum_{k=1}^K \hat{\mathcal{R}}_n(\mathcal{U}_k) .
$$
\end{lemma}

The following generalized Sauer-Shelah lemmas will be useful.
\begin{lemma}[After Corollary 5 in \cite{Haussler95} and Theorem 9 in \cite{BenDavid95}]
\label{lem:sauernatarajan}
Let $d_{\G}$ be the Natarajan dimension of $\G$. Then, for any $\g x_n\in\X^n$, 
$$
	|\G_{\g x_n}| \leq \sum_{i=1}^{d_{\G}}\binom{n}{i}\binom{C}{2}^i \leq \left( \frac{n e C}{2d_{\G}}\right)^{d_{\G}}.
$$
\end{lemma}
\begin{lemma}[After Lemma 3.5 in \cite{Alon97}]
\label{lem:alon}
For a class $\F$ of functions from $\X$ into $[-M,M]$, let $d_{\F}(\epsilon)$ denote its fat-shattering dimension at scale $\epsilon$. Then, for any $\epsilon\in(0,2M]$ and $n\in\N^*$, 
$$
	\mathcal{N}_{\infty}(\epsilon, \F, n) \leq 2\left(\frac{16M^2 n}{\epsilon^2}\right)^{d_{\F}(\frac{\epsilon}{4})\log_2\left(\frac{4M e n}{d_{\F}(\frac{\epsilon}{4})\epsilon}\right) } .
$$
\end{lemma}
\begin{lemma}[After Theorem 1 in \cite{Mendelson03} and Lemma~3 in \cite{Guermeur16}]
\label{lem:sauerdimensionfree}
For a class $\F$ of functions from $\X$ into $[-M,M]$, let $d_{\F}(\epsilon)$ denote its fat-shattering dimension at scale $\epsilon$. Then, for any $\epsilon\in(0,2M]$ and $n\in\N^*$, 
$$
	\mathcal{N}_2(\epsilon, \F, n) \leq \left(\frac{13M}{\epsilon}\right)^{20 d_{\F}(\frac{\epsilon}{96})} .
$$
\end{lemma}

For linear and/or kernel-based classes, the different capacity meaures can be bounded as follows.   
\begin{lemma}[After Theorem~4.6 in \cite{Bartlett99}]
\label{lem:fatlinear}
Given a Hilbert space $\H$ and a mapping $\phi:\X\rightarrow \H$, let $\X \subseteq \{x\in\X : \|\phi(x)\|_{\H} \leq R_x\}$ and 
$\F = \{f \in \R^{\X} : \ f(x) = \inner{w}{\phi(x)}_{\H},\ \|w\|_{\H} \leq R_{w}\}$. Then, for any $\epsilon >0$,  
$$
	d_{\F}(\epsilon) \leq  \left(\frac{R_x R_w}{\epsilon}\right)^2 .
$$
\end{lemma}

\begin{lemma}[After Theorem~5.5 in \cite{Mohri12}]
\label{lem:radlinear}
Given a Hilbert space $\H$ and a mapping $\phi:\X\rightarrow \H$, let $\X \subseteq \{x\in\X : \|\phi(x)\|_{\H} \leq R_x\}$ and 
$\F = \{f \in \R^{\X} : \ f(x) = \inner{w}{\phi(x)}_{\H},\ \|w\|_{\H} \leq R_{w}\}$. 
Then, for any $n\in \N^*$,
$$
	\mathcal{R}_n(\F) \leq \frac{R_xR_w}{\sqrt{n}} .
$$ 
\end{lemma}

\begin{lemma}[After Theorem~4 in \cite{Zhang02}]
\label{lem:zhanglinearinf}
Given a Hilbert space $\H$ and a mapping $\phi:\X\rightarrow \H$, let $\X \subseteq \{x\in\X : \|\phi(x)\|_{\H} \leq R_x\}$ and 
$\F = \{f \in \R^{\X} : \ f(x) = \inner{w}{\phi(x)}_{\H},\ \|w\|_{\H} \leq R_{w}\}$. 
Then, for any $\epsilon\in(0,R_x R_w)$,  
\begin{align*}
	\log \mathcal{N}_{\infty}(\epsilon, \F, n) &\leq 36  \frac{R_x^2 R_w^2}{\epsilon^2}  \log \left( 2\left\lceil \frac{4R_x R_w}{\epsilon} + 2\right\rceil n + 1\right) \\ 
	&\leq 36  \frac{R_x^2 R_w^2}{\epsilon^2}  \log \left(  \frac{15 R_x R_w n}{\epsilon} \right),  
\end{align*}
and for $\epsilon\geq  R_x R_w$, $\log \mathcal{N}_\infty(\epsilon, \F, n)=0$. 
\end{lemma}

\section{Proof of Theorem~\ref{thm:real}}
\label{sec:proofreal}

Theorem~\ref{thm:real} is a consequence of Theorem~\ref{thm:general} applied to the function class
\begin{equation}\label{eq:pwslossclass}
	\L_{p,\F} = \{ \ell  \in [0,1 ]^{\X\times\Y} : \ell(x,y) = |y - \bar{f}(x)|^p,\ \bar{f}\in \bar{\F} \}, 
\end{equation}
whose Rademacher complexity can be related to the one of $\mathcal{F}$ as follows.

Let us define the error class as $\mathcal{E} =  \{ e \in [-2M,2M]^{\X\times\Y} : e(x,y) = y - \bar{f}(x),\ \bar{f}\in\bar{\F} \}$. Define the function $\phi_p(u) = u^p$ with domain $[0, 2M]$ whose Lipschitz constant is $p(2M)^{p-1}$. Then, $\L_{p,\F} = \phi_p \circ |\cdot| \circ \mathcal{E}$ and, by contraction (see Lemma~\ref{lem:contraction} in Appendix~\ref{sec:tools}), we have
$$
	\mathcal{R}_n(\L_{p,\F}) = \mathcal{R}_n(\phi_p \circ |\cdot| \circ  \mathcal{E} ) \leq p2^{p-1}M^{p-1}\mathcal{R}_n(\mathcal{E} ) .
$$
Then, we bound the Rademacher complexity of the error class as
\begin{align*}
	\mathcal{R}_n(\mathcal{E} ) &= \E_{\g X_n \g Y_n \g \sigma_n} \sup_{\bar{f} \in \bar{\F}} \frac{1}{n} \sum_{i=1}^n \sigma_i (Y_i - \bar{f} (X_i ) ) \\
	&\leq \E_{\g Y_n \g \sigma_n} \frac{1}{n} \sum_{i=1}^n \sigma_i Y_i + \E_{\g X_n \g \sigma_n} \sup_{\bar{f} \in \bar{\F}} \frac{1}{n} \sum_{i=1}^n -\sigma_i \bar{f} (X_i ) 
\end{align*}
where 
$$
	\E_{\g Y_n \g \sigma_n} \frac{1}{n} \sum_{i=1}^n \sigma_i Y_i 
	=  \frac{1}{n} \sum_{i=1}^n \E_{Y_i \sigma_i} \sigma_i Y_i = 0
$$
and, since $\sigma_i$ and $-\sigma_i$ have the same distribution, 
$$
	\E_{\g X_n \g \sigma_n} \sup_{\bar{f} \in \bar{\F}} \frac{1}{n} \sum_{i=1}^n - \sigma_i \bar{f}(X_i ) 
	= \mathcal{R}_n(\bar{\F})
$$
Finally, by contraction, $ \mathcal{R}_n(\bar{\F})\leq  \mathcal{R}_n(\F)$ and taking $M=\frac{1}{2}$ completes the proof.

\section{Additional results on covering numbers}
\label{sec:addresults}

We need the following result on uniform covering numbers.

\begin{lemma}\label{lem:uniformcovmonotone}
Let $\F$ be a uniform GC class from $\X$ into $[-M,M]$. Then, for any $\epsilon>0$ and any $q\in[1,\infty)\cup\{\infty\}$, the uniform covering numbers $\mathcal{N}_q(\epsilon, \F, n)$ form a non-decreasing function of $n$.
\end{lemma}
\begin{proof}
Recall from \cite{Alon97} that a class of uniformly bounded real-valued functions is a uniform GC class if and only if its fat-shattering dimension is finite for all $\epsilon>0$. Note that this also implies the finiteness of its $L_q$-norm covering numbers by Lemma~\ref{lem:alon} and, for all $q\in[1,\infty)$, the relation $\mathcal{N}_q(\epsilon, \F, n)\leq \mathcal{N}_{\infty}(\epsilon, \F, n)$. Thus, for all $n$, $\mathcal{N}_q(\epsilon, \F, n)$ is finite and, since it is the largest of a finite set of integers, we have $\mathcal{N}_q(\epsilon, \F, n) = \max_{\g x_n\in \X^n} \mathcal{N}(\epsilon, \F, d_{q,\g x_n})$. Then, there is a sequence $\g x_n\in\X^n$ on which the maximum (let it be $N_n$) is attained, i.e., such that $N_n=\mathcal{N}_q(\epsilon, \F, n) = \mathcal{N}(\epsilon, \F, d_{q,\g x_n})$. 
Note that for any $\g x_n\in\X^n$, 
$$
	\mathcal{N}_q(\epsilon, \F, n+1) \geq \sup_{x\in \X} \mathcal{N}(\epsilon, \F, d_{q,\g x_n x }),
$$
where $d_{q,\g x_n x }$ is the pseudo-metric defined over the concatenation of the sequence $\g x_n$ with $x$. Therefore, it is sufficient to show that $\sup_{x\in\X} \mathcal{N}(\epsilon,\F,d_{q,\g x_n x}) \geq N_n$. 

For $q=\infty$, this is a direct consequence of the fact that for all $(f,h)\in\left(\R^{\X}\right)^2$,
$
	d_{\infty,\g x_n x}(f, h) = \max\{d_{\infty,\g x_n}(f, h), |f(x)-h(x)|\} \geq d_{\infty,\g x_n}(f, h) .
$
For $q\in[1,\infty)$, assume that it is not the case, then $\forall x\in \X$, there is an $\epsilon$-net $\H$ of $\F$ of cardinality $\mathcal{N}(\epsilon,\F,d_{q,\g x_n x}) < N_n$ and, for all $f\in\F$, there is an $h\in\H$ such that 
$$
	d_{q,\g x_n x}(f, h) \leq \epsilon . 
$$
Since 
$$
	d_{q,\g x_n x}(f, h)^q =  \frac{\sum_{i=1}^n |f(x_i) - h(x_i)|^q + |f(x) - h(x)|^q}{n+1}, 
$$
we have 
\begin{align}
	d_{q,\g x_n}(f, h)^q &= \frac{1}{n}\sum_{i=1}^n |f(x_i) - h(x_i)|^q \nonumber\\
		&= \frac{n+1}{n}\left(d_{q,\g x_n x}(f, h)^q - \frac{1}{n+1}|f(x) - h(x)|^q \right) \nonumber\\
		&\leq \frac{(n+1)\epsilon^q - |f(x) - h(x)|^q}{n} . \label{eq:dnp1}
\end{align}
In the case where $|f(x_i) - h(x_i)| < \epsilon$, for all $i \in [n]$, then, $d_{q,\g x_n}(f, h) = \left(\frac{1}{n}\sum_{i=1}^n |f(x_i) - h(x_i)|^q\right)^{\frac{1}{q}}\leq \epsilon$ and $\H$ is also an $\epsilon$-net of $\F$ for the pseudo-metric $d_{q,\g x_n}$, thus ensuring that $ \mathcal{N}(\epsilon, \F, d_{q,\g x_n}) \leq |\H|$, which contradicts the assumptions $\mathcal{N}(\epsilon, \F, d_{q,\g x_n}) = N_n$ and $|\H| < N_n$. If this is not the case, i.e., if there is some  $i\in [n]$, such that $|f(x_i) - h(x_i)|\geq \epsilon$, then  choosing $x=x_i$ in~\eqref{eq:dnp1} yields 
$$
	d_{q,\g x_n}(f, h)^q \leq  \frac{(n+1)\epsilon^q - \epsilon^q}{n} = \epsilon^q
$$
and therefore, $\H$ is also an $\epsilon$-net for the pseudo-metric $d_{q,\g x_n}$ in this case, showing again a contradiction. As a consequence, $\sup_{x\in\X} \mathcal{N}(\epsilon,\F,d_{q,\g x_n x}) \geq N_n$ and the lemma is proved.
\end{proof}

Using the ideas from the proof of Lemma 1 in \cite{Guermeur16}, we can derive the following structural result on covering numbers. Note that, for any $q\geq 1$, the dependency on $C$ of the bound in this lemma can be simplified by trivially upper bounding the covering number in $L_q$-norm by the one in $L_{\infty}$-norm.
\begin{lemma}
\label{lem:coveringmax}
Given a sequence of $C$ real-valued function classes $\mathcal{A}_k$ with domain $\Z$, let $\mathcal{A}$ be either the pointwise maximum class $\{a\in\R^{\Z} : a(z) = \max_{k\in[C]} a_k(z) ,\ a_k \in\mathcal{A}_k\}$  or the pointwise minimum class $\{a\in\R^{\Z} : a(z) = \min_{k\in[C]} a_k(z) ,\ a_k \in\mathcal{A}_k\}$. Then, for any $q\geq 1$ and $q=\infty$,
$$
	\mathcal{N}(\epsilon, \mathcal{A}, d_{q,\g z_n}) \leq \prod_{k=1}^C \mathcal{N}\left(\frac{\epsilon}{C^{1/q}}, \mathcal{A}_k, d_{q,\g z_n}\right) .
$$
\end{lemma}
\begin{proof}
We start with the pointwise maximum case and $q=\infty$. 
Let $\H_k$ be a minimal proper $\epsilon$-net of $\mathcal{A}_k$ and $\H$ be the pointwise maximum class of $(\H_k)_{k\in[C]}$ (and note that $\H_k\subseteq\mathcal{A}_k$ implies $\H\subseteq \mathcal{A}$). Let $k(f,z) = \argmax_k f_k(z)$. Then, for any $a\in\mathcal{A}$ and $h\in\H$, there are $(a_k)_{k\in[C]}$ and $(h_k)_{k\in[C]}$ such that
\begin{align*}
	d_{\infty,\g z_n}(a, h) = \max_{i\in[n]} |a_{k(a,z_i)}(z_i) - h_{k(h,z_i)}(z_i)|.
\end{align*}
By using the definition of $k(f,z)$, we can deduce that if $a_{k(a,z_i)}(z_i) \geq h_{k(h,z_i)}(z_i)$, then
\begin{align*}
	 |a_{k(a,z_i)}(z_i) - h_{k(h,z_i)}(z_i) | &=  a_{k(a,z_i)}(z_i) - h_{k(h,z_i)}(z_i) \\
	 & \leq a_{k(a,z_i)}(z_i) - h_{k(a,z_i)}(z_i)\\
	 &  \leq | a_{k(a,z_i)}(z_i) - h_{k(a,z_i)}(z_i)| 
\end{align*}
and that if $a_{a,k(z_i)}(z_i) < h_{k(h,z_i)}(z_i)$, then
\begin{align*}
	 |a_{k(a,z_i)}(z_i) - h_{k(h,z_i)}(z_i) | &=  h_{k(h,z_i)}(z_i)  - a_{k(a,z_i)}(z_i) \\
	 & \leq h_{k(h,z_i)}(z_i)  - a_{k(h,z_i)}(z_i) \\
	 &  \leq |h_{k(h,z_i)}(z_i)  - a_{k(h,z_i)}(z_i)|.
\end{align*}
Thus, 
$$
|a_{k(a,z_i)}(z_i) - h_{k(h,z_i)}(z_i) | \leq \max_{k\in [C]}| a_{k}(z_i) - h_{k}(z_i)|
$$
and
\begin{align}\label{eq:coveringmaxinf}
	d_{\infty,\g z_n}(a, h ) &\leq \max_{i\in[n]}\max_{k\in [C]}| a_{k}(z_i) - h_{k}(z_i)| \\
			&\leq \max_{k\in [C]} \max_{i\in[n]}| a_{k}(z_i) - h_{k}(z_i)|\nonumber\\
			&\leq\max_{k\in [C]}d_{\infty,\g z_n}(a_k, h_k ) \nonumber\\
			&\leq\max_{k\in [C]} \epsilon \nonumber\\
			&\leq \epsilon , \nonumber
\end{align}
which proves the statement for the pointwise maximum class. 

If $\mathcal{A}$ is the pointwise minimum class, let $\mathcal{A}'$ be the pointwise maximum of $(-\mathcal{A}_k)$, and note that $\mathcal{A} = -\mathcal{A}'$. Then the statement follows by using $\mathcal{N}(\epsilon, -\mathcal{A}', d_{p,\g z_n}) = \mathcal{N}(\epsilon, \mathcal{A}', d_{p,\g z_n})$, $\mathcal{N}(\epsilon, \mathcal{A}_k, d_{p,\g z_n}) = \mathcal{N}(\epsilon, -\mathcal{A}_k, d_{p,\g z_n})$ and the result for the pointwise maximum class $\mathcal{A}'$.

For $q\in[1,\infty)$, the same reasoning applies with~\eqref{eq:coveringmaxinf} replaced by
\begin{align*}
	d_{q,\g z_n}(a, h )^q &\leq \frac{1}{n}\sum_{i=1}^n \left(\max_{k\in [C]}| a_{k}(z_i) - h_{k}(z_i)| \right)^q \\
			&\leq \frac{1}{n}\sum_{i=1}^n \max_{k\in [C]}| a_{k}(z_i) - h_{k}(z_i)|^q \\
	 		&\leq \frac{1}{n}\sum_{i=1}^n \sum_{k=1}^C | a_{k}(z_i) - h_{k}(z_i)|^q \\
	 		&\leq \sum_{k=1}^C d_{q,\g z_n}(a_k, h_k )^q \\
			&\leq C \epsilon^q 
\end{align*}
and a rescaling of $\epsilon$. 
\end{proof}

\ifCLASSOPTIONcaptionsoff
  \newpage
\fi


\end{document}